\documentclass[runningheads]{llncs}
\usepackage{graphicx}
\usepackage{times}

\usepackage{soul}
\usepackage{url}
\usepackage[hidelinks]{hyperref}
\usepackage[utf8]{inputenc}
\usepackage[small]{caption}
\usepackage{graphicx}
\usepackage{amsmath}
\usepackage{booktabs}
\usepackage{helvet}
\usepackage{courier}

\usepackage{amsmath,amssymb,amsfonts}
\usepackage{algorithmic}
\usepackage{graphicx}
\usepackage{textcomp}
\usepackage{xcolor}

\usepackage{mathtools}
\usepackage[utf8]{inputenc} 
\usepackage[T1]{fontenc}    
\usepackage{hyperref}       
\usepackage{url}            
\usepackage{booktabs}       
\usepackage{amsfonts}       
\usepackage{nicefrac}       
\usepackage{microtype}      
\usepackage{lipsum}
\usepackage{multirow}
\usepackage{array,graphicx}
\usepackage[ruled,linesnumbered]{algorithm2e}
\usepackage{float}
\usepackage[english]{babel}
\usepackage{amsthm}

\begin{document}
\title{Strategic and Crowd-Aware Itinerary Recommendation}
\author{
Junhua Liu\inst{1,2} \and 
Kristin L. Wood\inst{1,3} \and 
Kwan Hui Lim\inst{1}
}


\authorrunning{J. Liu et al.}

%

\institute{
 Singapore University of Technology and Design \and
 Forth AI \and
 University of Colorado Denver \\
 \email{j@forth.ai, \{kristinwood, kwanhui\_lim\}@sutd.edu.sg}
}

\maketitle              
\begin{abstract}
There is a rapidly growing demand for itinerary planning in tourism but this task remains complex and difficult, especially when considering the need to optimize for queuing time and crowd levels for multiple users. This difficulty is further complicated by the large amount of parameters involved, i.e., attraction popularity, queuing time, walking time, operating hours, etc. Many recent works propose solutions based on the single-person perspective, but otherwise do not address real-world problems resulting from natural crowd behavior, such as the Selfish Routing problem, which describes the consequence of ineffective network and sub-optimal social outcome by leaving agents to decide freely. In this work, we propose the Strategic and Crowd-Aware Itinerary Recommendation (SCAIR) algorithm which optimizes social welfare in real-world situations. We formulate the strategy of route recommendation as Markov chains which enables our simulations to be carried out in poly-time. We then evaluate our proposed algorithm against various competitive and realistic baselines using a theme park dataset. Our simulation results highlight the existence of the Selfish Routing problem and show that SCAIR outperforms the baselines in handling this issue. 

\end{abstract}
\begin{figure}[h]
\centering
\includegraphics[width=\textwidth]{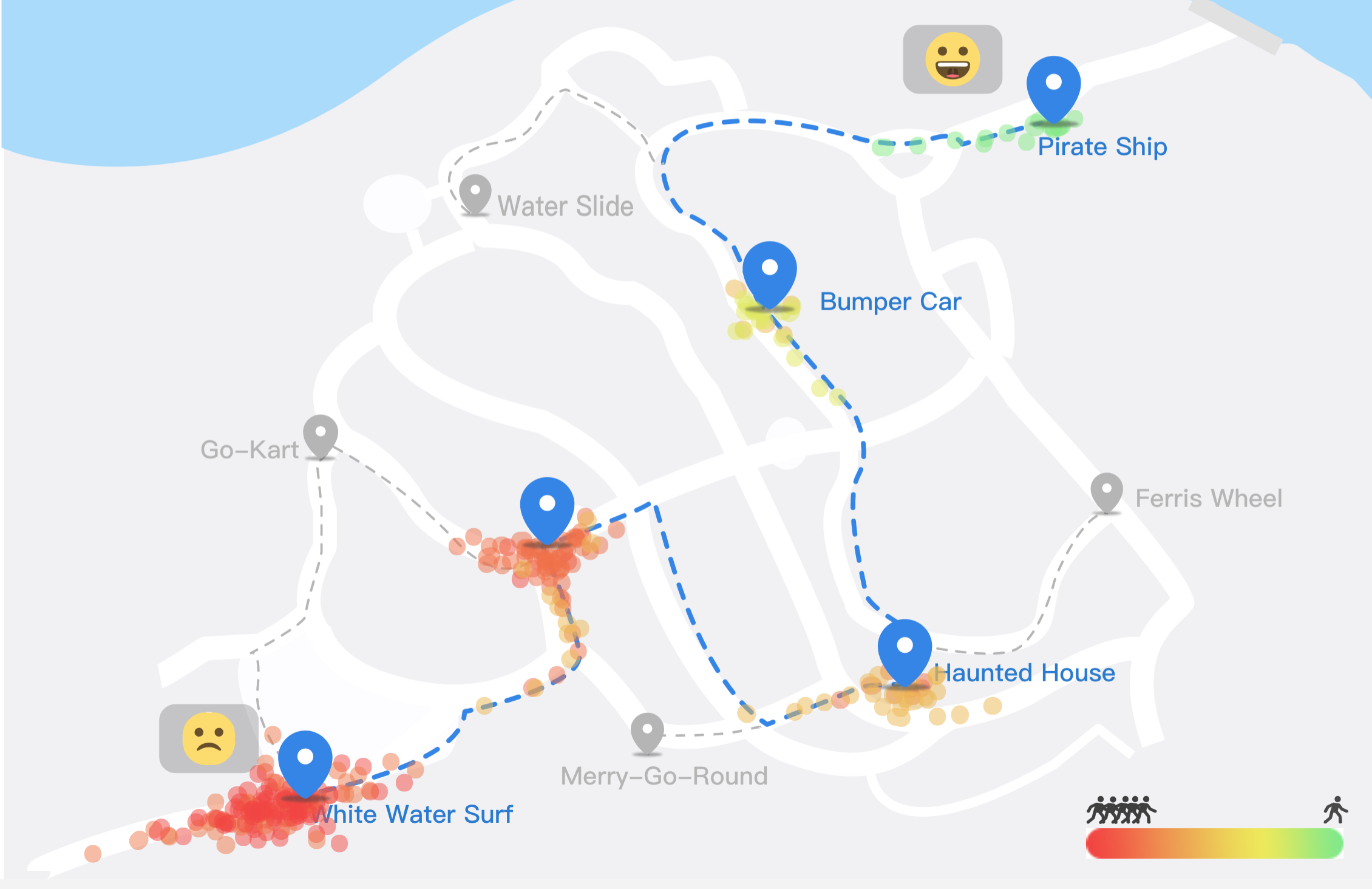}
\caption{Existing itinerary recommendation problems leverage data-driven approaches with a single-person perspective. In real life, this will result in the Selfish Routing problem, where leaving all agents free to act according to their own interests results in a sub-optimal social welfare. As illustrated, the recommended path performed sub-optimally, where the closer the POIs are to the start of the route, the more crowded they would be, while leaving all other POIs (in grey) not utilized.}
\label{fig:sr}
\end{figure}

\section{Introduction}
Itinerary recommendation has seen a rapid growth in recent years due to its importance in various domains and applications, such as in planning tour itineraries for tourism purposes. Itinerary recommendation and planning is especially complex and challenging where it involves multiple points of interest (POIs), which have varying levels of popularity and crowdedness. For instance, while visiting a theme park, the visitor's route can include POIs such as roller coasters, water rides, and other attractions or events. The itinerary recommendation problem can be modelled as an utility optimization problem that maximizes the number of facilities visited and the popularity of these facilities\footnote{The terms "POIs", "attractions" and "facilities" are used interchangeably.}, while minimizing the queuing time and travel time from one facility to the other. Facilities in a theme park come with different properties such as popularity, duration, location and a dynamic queuing time. Visitors are often constrained by a time budget that limits the number of facilities one could visit in a single trip. While many algorithms have been developed~\cite{zhang-tois16,lim-kais17,choudhury-ht10,gionis-wsdm14,Padia-BigData19}, they mostly aim to recommend itineraries for individual travellers, whereas a real-life itinerary is also affected by the actions of other travellers, such as lengthening the queuing time at a facility.


Many works focus on constructing a single optimal path for the individual traveller, solely based on historical data. While this approach works for the individual traveller, it leads to an sub-optimal itinerary when all travellers are given the same recommendation. Consider a recommender system that recommends an itinerary comprising the most popular POIs with the least queuing time based on such historical data. In a real-life scenario with multiple travellers,  all travellers will follow the same recommended itinerary with the shortest historical queuing time, resulting in an expected queuing time that would grow with each new arrival, as illustrated in Figure~\ref{fig:sr}. In other words, the later an agent\footnote{We use the terms "travellers", "visitors" and "agents" interchangeably.} arrives to the system, the longer her expected queuing time will be. As a result, the social welfare or the collective utility of all agents has failed to be optimized. As an individual traveller, it is extremely difficult for an agent to gain knowledge of the system state, i.e., the people who are visiting the park and their respective paths. As a result, letting the agent find an optimal strategy that maximizes her expected utility is unrealistic without considering the actions of other agents. 

To address this problem, we propose the Strategic and Crowd-Aware Itinerary Recommendation (SCAIR) algorithm, which is a recommender system that maintains an internal information of all recommended routes and leverages on this internal information to make routing recommendations to its arriving agents. In other words, we take a game-theoretic approach to address the problem and formulate a crowd-aware itinerary recommendation algorithm having in mind the Selfish Routing problem~\cite{roughgarden2005selfish}, i.e., allowing agents act freely results in a sub-optimal social welfare. Concretely, we model the itinerary recommendation problem into a strategic game~\cite{osborne1994course}, where the system, i.e., a theme park, defines a set of allocation rules to allocate route to each player in the system, instead of leaving the agents a high degree of freedom to choose their own path. Experiments show that our approach is effective in optimizing utility of all agents. 

\section{Main Contributions}
Our main contributions are as follows:
\begin{itemize}
    \item We introduce and formulate the crowd-aware itinerary recommendation problem as a social welfare optimization problem that considers the actions of multiple travellers, in contrast to existing works that only consider the perspective of the single traveller (Section~\ref{problemDef}).
    \item To address this crowd-aware itinerary recommendation problem, we propose the SCAIR algorithm which utilizes a game-theoretic approach to recommend itineraries for multiple travellers (Section~\ref{proposedAlgo}).
    \item Using a theme park dataset, we compare our SCAIR algorithm against various competitive and realistic baselines and show how SCAIR outperforms these baselines with a large reduction in queuing times (Sections~\ref{experimentSetup} and~\ref{results}).
\end{itemize}

For the rest of the paper, Section~\ref{litReview} discusses related works and how our research differs from these earlier works. Section~\ref{conclusion} summarizes this paper and introduces some future research directions.
Next, we introduce the problem formulation of this crowd-aware itinerary recommendation problem.


\section{Crowd-aware Itinerary Recommendation Problem}
\label{problemDef}

\subsection{General Approach}
In this work, we view the itinerary recommendation problem from a global perspective and formulate it as a strategic game where the system designs and distributes the optimal path to every agent on arrival, based on the existing agents in the system and their respective paths. In the context of a theme park, one can think of this entity as the theme park operator that gives out the recommendation of various itineraries to visit the attractions to different visitors. We propose the SCAIR algorithm that dynamically recommends routes taking into consideration all existing agents in the system. 

The crowd-aware itinerary recommendation problem aims to maximize the sum of all agents' utility in the system. This turns out to be a social welfare optimization problem that is NP-hard~\cite{nguyen2013survey}. Furthermore, simulating or solving the problem is also empirically challenging. One has to take into consideration the entire history of existing visitors results in exponential space-complexity with respect to the number of agents, and exponential time-complexity with respect to the number of facilities in a path. 

To overcome these challenges, we propose a simplified version which models the recommendation problem as a finite markov chains and is known to be in NC~\cite{papadimitriou1987complexity} and decidable in poly-logarithmic time~\cite{arora2009computational}. The simplified model makes an assumption that each decision embeds information of the immediate last decision and the model as a result is able to provide a snapshot of the entire history. Next we will discuss the formulation of the problem. 

\subsection{Problem Formulation}
We formulate the crowd-aware itinerary recommendation problem to be a finite markov chain and impose constraints such as (1) fixing the starting point, (2) setting a time budget for the path, and (3) limiting the distance between two stations. These constraints reflect real-life considerations closely, such as fixed starting point near the entrance; visitors having limited time to tour; and dissatisfaction arising with long walking distance among facilities.

Concretely, we model the theme park comprising numerous tourist attractions as a fully connected graph $G(F, C)$, where $F = \{f_1, ..., f_n\}$ is the collection of $n$ facilities in the system, and $C = [c_{ij}]$ is the set of connections from $f_i$ to $f_j$. Each connection $c_x$ is associated with the properties of distance $Dist(c_{ij})$ and travel time $Trav(c_{ij})$ in minutes. Each facility $f_x$ is associated with a set of properties including coordinates $(lat_x, long_x)$, duration of visit $Dur(f_x)$ in minutes, capacity $Cap(f_x)$ and popularity $Pop(f_x)$. 

We formulate the agents' visits as $m$ states $S=\{s_1, ..., s_m\}$, where each state $s_x$ is associated with a feasible path $p_x = [f^{(x)}_1, ..., f^{(x)}_{n_x}]$ with $n$ facilities $[f^{(x)}_1, ..., f^{(x_n)}]$. The total time $TT_x$ of path $p_x$ is defined as:
\begin{equation}
\label{eq:ttx}
TT_x = \sum_{i=1}^{n_x}{Dur(f^{(x)}_{i})} + \sum_{i=1}^{n_x-1}{Trav(c_{i,i+1})}
\end{equation}

We model the utility of the agents with respect to the popularity of each facility visit normalized by the expected waiting time at each facility. Our assumption is that higher popularity of a facility indicates a greater attractiveness to visitors, subjected to how long they have to wait for that facility. Concretely, we define the utility function $U_x$ for path $x$ with $n$ nodes as follows:

\begin{equation}
\label{eq:util}
U_x = \frac{\sum_{f \in p_j}{Pop(f)}}{Q(p_{x}|p_{x-1})}
\end{equation}
where $Q(p_{x}|p_{x-1})$ is the expected queuing time at path $p_x$ given $p_{x-1}$, and $Pop(p_x)$ is the sum of popularity of all facilities in the path. The path's expected queuing time $Q(p_{x}|p_{x-1})$ is calculated by summing up the queuing time at all facilities:
\begin{equation}
Q(f_i)= \frac{1}{Cap(f_y)}Dur(f_y)  \delta{(f^{(x)}_{y,h} = f^{(x-1)}_{y,h})}    
\end{equation}
where $\delta{(f^{(x)}_{y,h} = f^{(x-1)}_{y,h})}=1$ if the facility appears to overlap between paths $p_x$ and $p_{x-1}$ within the same hour $h$. Capacity $Cap(f_x)$ is set to be a constant for simplicity. Finally, the transition matrix $T$ is defined as:
\begin{equation}
    T_{ij} = \frac{\sum_{f \in p_j}{Pop(f)}} {Q(p_{j}|p_{j-1=i} )}
\end{equation}

The transition matrix is then normalized by:
\begin{equation}
    T_{ij} := \frac{T_{ij}}{\sum_j{T_{ij}}} 
\end{equation}

The set of feasible paths, i.e., total search space, is determined by solving an optimization problem, as follows:

\begin{equation}
\label{eq:pf}
\begin{aligned}
& \text{maximize} 
& & TT_{x} = \sum_{i=1}^{n_x}{Dur(f_i)} + \sum_{j=1}^{n_x-1}{Trav(c_{j, j+1})} \\
& \text{subject to} 
& & Dist(c_{j, j+1}) \leq{s}, \;\; TT_{x}\leq{t} \\
\end{aligned}
\end{equation}
for $n$ facilities in the path, with a constant time budget $t$.\\

Finally, we model the strategic itinerary recommendation problem as a social welfare optimization problem as follows:

\begin{equation}
\label{eq:scair}
\begin{aligned}
& \text{maximize} 
& & W = \sum_{x}U_xp_x \\
& \text{subject to} 
& & \sum_{x}TT_x \leq{t}, \;\; x\in{\{1,...,n\}}
\end{aligned}
\end{equation}
for $n$ agents and time budget $t$.

\subsection{Proof of NP-Hardness}
We further investigate the NP-hardness of various sub-problems and show the respective proofs in this section.

\begin{theorem}
\label{thm:pf}
The path finding problem defined in Equation~\ref{eq:pf} is NP-hard.
\end{theorem}

\begin{proof}
We prove the NP-hardness of the path finding problem by reduction from the 0-1 Knapsack problem which is known to be NP-hard~\cite{martello1999dynamic}. Recall that the 0-1 Knapsack problem is a decision problem as follows:

\begin{equation}
\label{eq:knapsack}
\begin{aligned}
& \text{maximize} 
& & z = \sum_{i}p_ix_i \\
& \text{subject to} 
& & \sum_{i}w_ix_i \leq{c} \\
&
& & x_i\in{\{0,1\}}, \;\; i\in{\{1,...,n\}}
\end{aligned}
\end{equation}

for $n$ available items where $x_i$ represents the decision of packing item $i$, $p_i$ is the profit of packing item $i$, $w_i$ is the weight of item $i$, $c$ is the capacity of the knapsack. 

Intuitively, the path finding problem is a decision problem of allocating a set of facilities into a path with a capacity of time budget, where each facility comes with properties of profit and duration time.

Formally, we transform the minimization problem in Equation~\ref{eq:pf} to an equivalent maximization problem. Concretely, the binary variable $f_i\in{\{0,1\}}$ is included, where $f_i = 1$ if $f_i$ is in path $p_x$, and 0 if otherwise. Furthermore, we define the profit of facility $f_i$ as $p_i = -Dur(f_i)$ and set the travel time $Trav(c_{ij})$ to be a constant. Finally, the distance constant cap $s$ is set to be infinity. The new problem formulation is represented as follows: 

\begin{equation}
\label{eq:newpf}
\begin{aligned}
& \text{maximize} 
& & T_{path}' = \sum_{i}p_if_i \\
& \text{subject to} 
& & \sum_{i}{Dur(f_i)f_i}\leq{t} \\
&
& & f_i\in{\{0,1\}}, \;\; i \in{\{1,...,n\}}
\end{aligned}
\end{equation}

In this formulation, a path is equivalent to the knapsack in the 0-1 Knapsack problem, where each facility has its profit of $p_i)$, and its cost of $Dur(f_i)$ that is equivalent to the profit and weight of an item respectively. The maximization problem is subjected to a constant time budget $t$ which is equivalent to the capacity $c$ in a 0-1 Knapsack problem. 

As a result, for any instance of the 0-1 Knapsack problem (i.e. item allocation decisions), we are able to find an equivalent instance of the path finding problem (i.e. a facility allocation decisions). Therefore, a solution in the path finding problem yields an equivalent solution to the 0-1 Knapsack decision problem. As such, we have completed the proof of NP-hardness for our path finding problem to be NP-hard. 

\end{proof}

\begin{theorem}
\label{thm:welfareop}
The social welfare optimization problem defined in Equation~\ref{eq:scair} is NP-hard.
\end{theorem}

\begin{proof}
Once again, we prove the NP-hardness of our welfare optimization problem by reduction from the 0-1 Knapsack problem.

In Equation~\ref{eq:scair}, the set of paths assigned to agents in the system is equivalent to the set of items in 0-1 Knapsack problem; each path has its utility and total time, which are equivalent to the profit and weight of an item respectively; the maximization problem is subjected to a constant time budget $t$ which is equivalent to the capacity $c$ in a 0-1 Knapsack problem. 

As a result, for any instance of the 0-1 Knapsack problem decisions, we are able to find an equivalent instance of a path assignment decision that yields a solution to the original Knapsack decision problem. As such, we conclude the proof of NP-hardness and have shown that our welfare recommendation problem is NP-hard. 

\end{proof}

Next, we describe our proposed SCAIR algorithm for solving crowd-aware itinerary recommendation problem.

\section{Strategic and Crowd-Aware Itinerary Recommendation (SCAIR) Algorithm}
\label{proposedAlgo}

In this section, we describe our proposed SCAIR algorithm, which comprises the main steps of finding feasible paths, generating a transition matrix and simulating traveller visits. 

\subsection{Finding Feasible Paths}
\label{findFeasiblePaths}
Algorithm~\ref{algo:feasible-path} shows the pseudocode of our path finding algorithm based on a breadth-first strategy. The input is a graph $G(F,C)$ that represent a theme park with the set of facilities $F$ and connections $C$, time budget $TT_{max}$, and distance limit between two facilities $Dist_{max}$. This algorithm then generates and returns a collection of feasible paths, $Paths$, with respect to the provided input graph $G(F,C)$. 

\begin{algorithm}[t]
 \KwData{$f_i \in{F},  c_{ij}\in{C},  TT_{max},  Dist_{max},  f_0$}
 \KwResult{$Paths$: the set of feasible paths}
    \SetAlgoLined
    \Begin{
        $Paths = [[f_0]]$\;
        \While{True}{
            \For {$path_i \in Paths$}{
                $VF = FindViableFacilities(f^{(i)}_{-1}, Dist_{max})$\;
                \If{$len(VF)==0$}{
                    $path_x = path_i + [FindNextNearest(f^{(i)}_{(-1)})]$\;
                    \If{$TT_x < TT_{max}$ and $path_x \not \in Paths$}{
                        $Paths += [path_x] $\;
                        $Paths.pop(path_i)$
                    }
                }
                  {
                    \ForEach{$vf \in VF$}{
                        $path_{x} = path_i + [vf]$\;
                        \If{$TT_x < TT_{max}$ and $path_x \not\in Paths$}{
                            $Paths += [path_x]$\;
                        }
                    }
                    $Paths.pop(path_i)$\;
                }
            }
            \If{$AllPathsMaxTimeBudget(Paths)$ or $AllPathsReachFullLength(Paths)$}{break\;}
        }        
    }
 \caption{SCAIR - FindFeasiblePaths()}
 \label{algo:feasible-path}
\end{algorithm}

We iterate the collection of intermediate $Paths$, and call the $FindViableFacilities$ function to find viable facilities, where $f^{(i)}_{-1}$ is the last facility of the path, and $Dist_{max}$ is the maximum distance an agent wants to travel from one facility to another. We set the parameters of total time budget $T_{max} < 8 hours$ and maximum allowed distance between two facilities $Dis_{max}(f_current, f_next) < 200m$. If there are no available facility that meets the distance constraint and the path has available time budget remaining, the agent proceeds to the next nearest facility. We also do not allow an agent to revisit a facility in the same trip.

\begin{algorithm}
 \KwData{$Parks, TimeBudgets, ArrivalIntervals$}
 \KwResult{Export simulation data to a CSV file}
    \SetAlgoLined
    \Begin{
        $Results = \{\}$\;
        \For{$Park \in Parks$}{
            \For{$SimTime \in TimeBudgets$}{
                \For {$\lambda \in ArrivalIntervals$}{
                    $Paths = FindFeasiblePaths(Park, SimTime)$\;
                    $T = ConstructTM(Park, Paths)$\;
                    $Qt, Pop, Utility = RunSimulation(Paths, \lambda, SimTime)$\;
                    $Update(Results, [Qt, Pop, Utility])$\;
                }
            }
        }
        $ExportCsvFromDict(Results)$\;
    }
 \caption{SCAIR - Simulate()}
 \label{algo:simulate}
\end{algorithm}

\textbf{Line 2}. The algorithm starts with constructing a 2-dimensional array, where each row represent a path as a sequence of facilities visited. We then conduct a breadth-first search (line 3 to 25), starting with the first row with an element of the initial facility, i.e., the entrance of a theme park. 

\textbf{Line 6 to 11}. If the algorithm is unable to find a facility within the feasible range, it will instead find the nearest facility that is not yet visited, and assign the new path into the $Paths$ collection if two conditions are met, namely (1) the new path's total time is within the visitor's time budget $TT_{max}$, and (2) no identical path exists in the $Paths$ collection. Eventually we remove the path the iteration started off.

\textbf{Line 13 to 20}. If the algorithm manages to find a set of viable facilities, it will then iterate through the set and execute a similar selection process.

\textbf{Line 22 to 24}. The algorithm breaks out from the infinite loop when any one of two conditions is met, namely (1) all paths in the $Paths$ collection have maximized its time budget i.e. any additional facility will make the total time of a path to be larger than the visitor's time budget; or (2) every path has included all available facilities.

\subsection{Transition Matrix} 
\label{transMatrix}
Using the set of feasible paths found (Section~\ref{findFeasiblePaths}), we now construct a Transition Matrix $T$ by calculating $T_{ij}$ as the costs of taking path $j$ given path $j-1=i$. The output of $FindCost()$ function varies based on the arrival interval $\lambda$ because it affects the expected time of arrival for each facilities at $path_j$, which leads to different occurrence of overlapping facilities between $path_i$ and $path_j$. 

\subsection{Simulation} Algorithm~\ref{algo:simulate} shows an overview of the simulation procedure, which involves iterating through the visit data of theme parks $Parks$, a list of time budgets $TimeBudgets$, and an array of arrival intervals $ArrivalIntervals$. 

\textbf{Line 6 to line 13}. For each step, the $FindFeasiblePaths()$ function finds the set of feasible paths which enables the $ConstructTM()$ function to construct the transition matrix, with input parameters namely park data $Park$ and simulation time $SimTime$. The $RunSimulation()$ function then runs the simulation to find the total queuing time $Qt$, average sum of popularity among all facilities visited $Pop$, and the expected utility $Utility$ which is calculated as a function of $Qt$ and $Pop$. Finally, we update the $Results$ dictionary (Line 9) and export the experimental data into CSV files (line 13) after completing the simulations. 

\section{Experimental Setup}
\label{experimentSetup}

In this section, we describe our dataset, evaluation process and baselines.

\subsection{Dataset}

We conduct our experiments using a publicly available theme park dataset from~\cite{lim2017personalized}. This dataset is based on more than 655k geo-tagged photos from Flickr and is the first that includes the queuing time distribution of attractions in various Disney theme parks in the United States.
In our work, we perform our experiments and evaluation using the data of visits in Epcot Theme Park and Disney Hollywood Studio.

\subsection{Experimental Parameters}
As previously described in Section~\ref{transMatrix}, we denote the arrival interval of agents as $\lambda$ which indicates the time between the arrival of two agents, measured in minute. In this work, $\lambda$ is set to be a constant for simplicity.  
For a robust evaluation, we perform our evaluation using multiple values of the evaluation parameters, namely arrival interval $\lambda\in{\{0.01, ... 0.09, 0.1, ..., 1.0\}}$, and simulation time $T$ between 60 and 360 minutes in 30 minutes intervals (i.e. $T\in{\{60, 90, ..., 360\}}$). 

\subsection{Evaluation and Baselines}
We compare our proposed SCAIR algorithm against three competitive and realistic baselines. The first two algorithms are based on intuitive strategies commonly used by visitors in real-life~\cite{lim-kais17}, while the third is a greedy algorithm used in~\cite{zhang2015personalized}. In summary, the three baseline algorithms are:
\begin{enumerate}
    \item Distance Optimization (denoted as $DisOp$)~\cite{lim-kais17}. An iterative algorithm where agents always choose the facility with the shortest distance to the currently chosen one.
    \item Popularity Optimization (denoted as $PopOp$)~\cite{lim-kais17}. An iterative algorithm where agents always choose the next most popular facility that satisfies the specified distance constraint from the currently chosen one.
    \item Popularity over Distance Optimization (denoted as $PodOp$)~\cite{zhang2015personalized}. An iterative greedy approach that models utility as the popularity of the POI normalized by the distance from the current one, and iteratively chooses the POI with the highest utility.
\end{enumerate}

Similar to many itinerary recommendation works~\cite{lim-kais19,lim2017personalized}, we adopt the following evaluation metrics:
\begin{enumerate}
    \item Average Popularity of Itinerary (denoted as $Avg Pop$). Defined as the average popularity of all attractions recommended in the itineraries.
    \item Expected Queuing Time per Visitor (denoted as $Avg Qt$). Defined as the average queuing time that each visitor spends waiting for attractions in the recommended itinerary.
    \item Expected Utility (denoted as $Uty$). Defined as the average utility score for all users based on the recommended itineraries.
\end{enumerate}


\begin{figure}[pth]
\centering
\includegraphics[width=\textwidth]{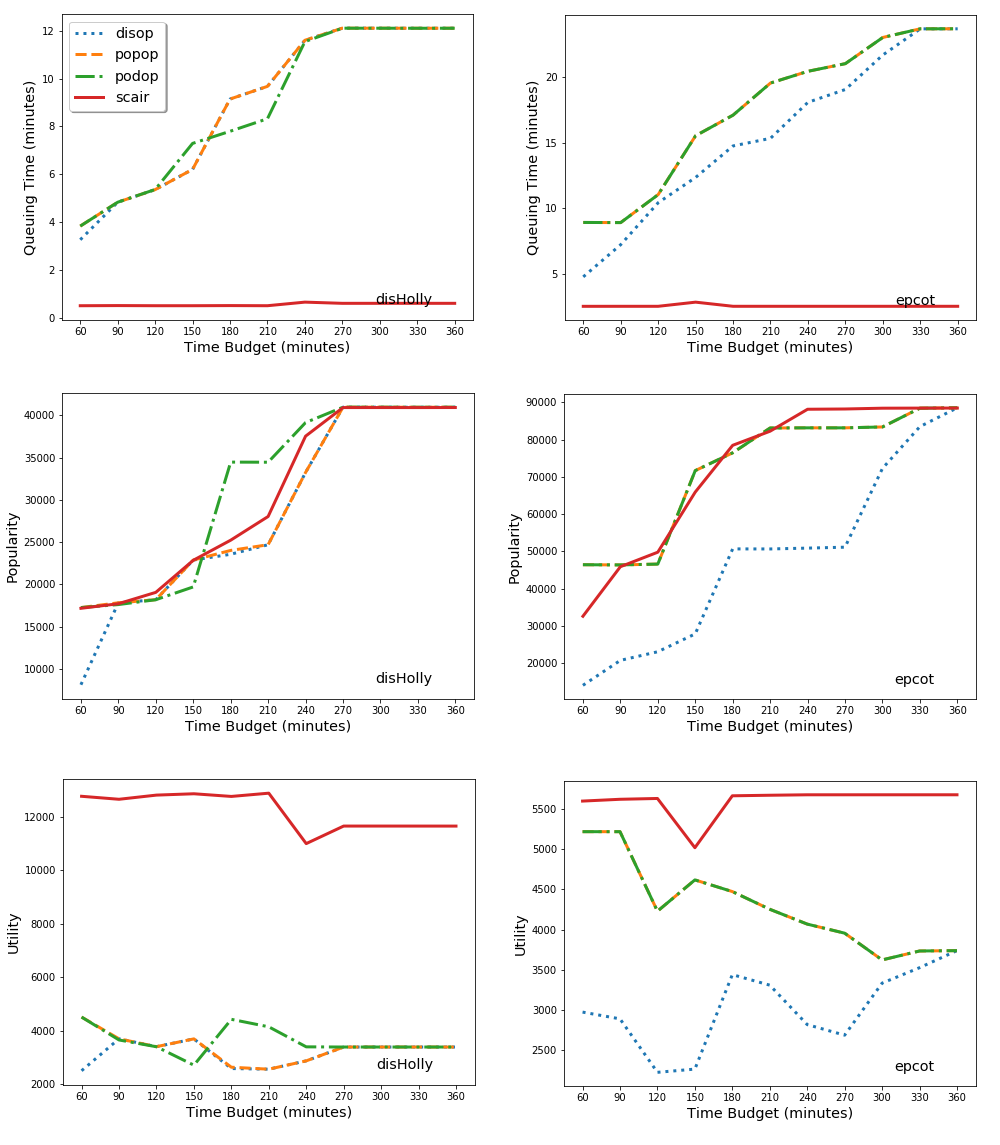}
\caption{The plots show how the queuing times, popularity and utility change with respect to simulation time $T$, over two theme parks data (Disney Hollywood and Epcot Theme Park). We observe that: (1) SCAIR's queuing time is consistently and significantly lower than the baselines. (2) Popularity of of all 4 algorithms perform similarly for DisHolly, while DisOp performs significantly poorer than the others for Epcot. (3) SCAIR's utility consistently outperforms the baselines.}
\label{fig:qt}
\end{figure}

\section{Results and Discussion}
\label{results}

Figure~\ref{fig:qt} shows the experimental results of the SCAIR algorithm compared to the three baseline algorithms. The x-axis indicates the time budget of visits and the y-axis indicates the queuing time, popularity and utility. Multiple experiments are conducted based on different arrival intervals $\lambda$, i.e., from 0.01 to 0.1 with a step size of 0.01, and from 0.1 to 1.0 with a step size of 0.1. The values in the graph are averaged across all $\lambda$.

\bgroup
\def\arraystretch{1.5}
\begin{table}[h] \centering
\caption{Queuing Time Ratio (Smaller values are better)}
\label{tab:qtratio}
\begin{tabular}{lcc}
\hline
 & Disney Hollywood & Epcot Theme Park \\
 & (DisHolly) & (Epcot) \\ \hline
DisOp & $0.045\pm0.221$ & $0.076\pm0.414$ \\
PopOp & $0.046\pm0.215$ & $0.092\pm0.368$ \\
PodOp & $0.045\pm0.211$ & $0.092\pm0.368$ \\
\color{blue}\textit{SCAIR} & \color{blue}\textit{$0.003\pm0.010$} & \color{blue}\textit{$0.016\pm0.006$} \\ \hline
\end{tabular}
\end{table}
\egroup

\textbf{Queuing Time.} In relative terms, we observe that SCAIR outperforms the baselines for both the queuing time and utility in both theme parks. SCAIR is able to maintain a low queuing time with different time budgets, while the baseline's queuing time increases with the growth of time budget. The observation is consistent for both theme parks. Table~\ref{tab:qtratio} shows the ratio of queuing time and time budget of visitors. SCAIR produces a queuing time ratio that is 78.9\% to 93.4\% shorter than that of the baselines, across both DisHolly and Epcot theme parks.

\textbf{Popularity.} All four algorithms perform similarly for the DisHolly dataset, while PopOp, PodOp and SCAIR remain similar but outperform DisOp for the Epcot dataset. We observe that PodOp achieves a relatively high Popularity when time budget is equal to 180 min and 210 min. We observe that this phenomena is due to the special geographic distribution of the POIs in DisHolly, where the optimal path according to the algorithm includes two POIs that are remote from other POIs but yield very high popularity. 

\textbf{Utility.} In terms of Utility, SCAIR outperforms all baselines consistently across all time budgets for both theme parks. The main contributing factor for this result is due to the much improved queuing time performance that SCAIR achieves, compared to the baselines. 

\section{Related Work}
\label{litReview}

Prior works propose different approaches for implementation to solve the itinerary recommendation problem. In the Information Retrieval community, many works use matrix factorization or collaborative filtering approaches to find a ranked list of top locations, which is known as top-k POIs recommendation~\cite{ye-sigir11,yuan-sigir13,li-sigir15,yao-sigir15,leung-sigir11}. In the Operations Research community, researchers have proposed heuristic approximation~\cite{zhang-tois16}, a modified Ant Colony System~\cite{wang-cikm16}, integer programming~\cite{lim-kais17} and similar methods to solve this itinerary recommendation problem.

Many works have modelled the itinerary recommendation problem as a variant of the Orienteering problem~\cite{choudhury-ht10,gionis-wsdm14,lim-kais19}. In the Orienteering problem, the recommendation aims to optimize social welfare with a global reward such as popularity, with respect to budget constraints such as travel time or distance among attractions in an itinerary. This approach typically does not take into consideration the trade-off between the duration in a facility and its popularity, which may contribute substantially to the global profit.

\subsection{Discussion}
These earlier works face a major limitation where the recommendation algorithms are constructed based on a single person's perspective. Despite some recent works exploring the effects of group or crowd behavior~\cite{wang-cikm16,garcia-esa11,lim-icaps16}, the algorithms treat the system as a static environment where properties such as queuing time only depend on historical data. Simulating an optimal path in such a static environment has a natural disadvantage where self-interested agents prioritize personal objective functions which may result in ineffective social welfare.  For instance, when everyone visiting the theme park follow the same recommended path, the queuing time will increase dramatically, and the optimality of such recommendation algorithms will then collapse. Roughgarden's work ~\cite{roughgarden2005selfish} discusses this problem extensively, defined as Selfish Routing problem, where giving agents freedom to act according to their own interests results in a sub-optimal social welfare.

The Selfish Routing problem was studied in the area of Game Theory and Mechanism Design~\cite{roughgarden2005selfish,koutsoupias1999worst,papadimitriou2001algorithms}. The inefficiency of achieving the optimize natural objective is quantitatively measured by Price of Anarchy, which was first defined as the ratio between the worst-case Nash equilibrium and the optimum sum of payoffs in game-theoretic environments~\cite{koutsoupias1999worst,papadimitriou2001algorithms}. Braess's Paradox for traffic flow~\cite{braess1968uber} describe the phenomenon where adding a new link to a transportation network might not improve the operation of the system, in the sense of reducing the total vehicle-minutes of travel in the system~\cite{eric1997braess}. To break out from this phenomenon, a system operator can manually interfere with or change agents' actions to provide policies or economic incentives with well designed strategies. Our proposed game-theoretic, dynamic itinerary recommendation algorithm in this paper is an instance of such strategy.

To address these limitations, we propose the Strategic and Crowd-Aware Itinerary Recommendation (SCAIR) algorithm to address the ineffectiveness of welfare optimization due to the lack of centralized control~\cite{Piliouras:2016:RSP:3007189.2930956}. The proposed recommendation algorithm takes into consideration all visits in a an itinerary planning scenario (e.g., a theme park), and makes recommendations to the next visitor with the knowledge of other visitors' paths in the park. Furthermore, the queuing time at all facilities at a certain hour is dynamically modelled according to the expected number of visitors in the same place at the same hour. 

\section{Conclusion and Future Work}
\label{conclusion}

\subsection{Conclusion and Discussion}
Prior works on itinerary recommendation typically aim to make recommendations for the individual traveller and perform poorly in scenarios where multiple travellers use the same recommended itinerary, i.e., the Selfish Routing problem. In this paper, we introduced the crowd-aware itinerary recommendation problem and highlighted this Selfish Routing problem where all self-interested agents aim to maximize their own utility which result in sub-optimal social welfare. For example, when all travellers are recommended the same POIs with a short queuing time based on historical data, those POIs then become congested and suffer from a long queuing time.

To address this problem, we proposed the SCAIR algorithm that takes into consideration crowd behavior and addresses the NP-hard Social Welfare Optimization problem with an finite markov chains, which is in NC and can be solved in poly-logarithmic time. We performed a series of experiments using a theme park dataset. Experimental results show that SCAIR outperforms various competitive baselines in terms of a reduced queuing time and improved utility, while offering similar levels of popularity scores.




\subsection{Future Work} 
We will further investigate models that further simulate real-life situations. For instance, we can also locate the entrances and exits of the theme parks to initialize and end paths; we could also use soft-max instead of one-hot to simulate the choices of paths which simulates the probabilistic decisions visitors make in real-life. We will also attempt to improve the formulation of the multi-objective optimization problem, such as by assessing the Pareto efficiency of the two objectives. It is also worthwhile looking into modifying our strategic recommendation algorithm and apply to game-theoretic environments, such as knowledge acquisition~\cite{liu2020self}, crisis management~\cite{liu2020crisisbert} and career path planning~\cite{dave2018combined,liu2019ipod}. Finally, we intend to look further into prior works, such as~\cite{lees2008does,palumbo2017predicting} to explore Machine Learning approaches in solving time-variant path planning problems and attempt to enhance the solution and simulation performance.  

\section{Acknowledgements}
This research is funded in part by the Singapore University of Technology and Design under grant SRG-ISTD-2018-140.

%
%
%
\bibliographystyle{splncs04}
\bibliography{ref}

\end{document}